\documentclass[11pt]{article}
\usepackage[margin=1in]{geometry} 

\usepackage{graphicx}
\usepackage{multirow} 
\usepackage{array,booktabs,arydshln,xcolor}
\usepackage[ruled,linesnumbered]{algorithm2e}
\usepackage{enumerate}
\usepackage{enumitem}
\usepackage{float}
\usepackage{tikz}
\usetikzlibrary{arrows, shapes}

\usepackage{caption}
\usepackage{mathtools}
\usepackage{hyperref}
\usepackage{fancyhdr}
\usepackage{relsize}

\usepackage{amsfonts}
\usepackage{amsmath}
\usepackage{bbm}
\usepackage{bm}
\usepackage{adjustbox}
\usepackage{amsthm, amssymb} 

\definecolor{toc}{RGB}{13,55,174}	
\usepackage{hyperref}				
\hypersetup{
    colorlinks=true,
    citecolor=red,
    filecolor=black,
    linkcolor=toc,
    urlcolor=toc
}

\allowdisplaybreaks 

\usepackage{thm-restate}
\newtheorem{theorem}{Theorem}[section]
\newtheorem{lemma}[theorem]{Lemma}
\newtheorem{corollary}{Corollary}[theorem]
\newtheorem{remark}{Remark}[theorem]

\newcommand{\lp}{\left} 
\newcommand{\rp}{\right}
\renewcommand{\Pr}[2]{\textbf{Pr}_{#1}\lp[#2\rp]} 
\newcommand{\E}[1]{\mathbb{E}\lp[#1\rp]}

\newcommand{\step}{\mathcal{T}}
\newcommand{\scenario}{\mathcal{S}}
 
\newtheorem{definition}{Definition}
\newcommand{\ind}[1]{\mathbbm{1}{ \{ #1 \} }}
\newcommand{\norm}[2]{|| #1 ||_{#2}}

\newcommand{\boxes}{\mathcal{B}}
\newcommand{\opt}{\text{OPT}}
\newcommand{\alg}{\text{ALG}}
\newcommand{\scenarios}{\mathcal{S}}
\newcommand\timeset{\mathcal{T}}
\newcommand{\pb}{\textsc{Pandora's Box}}
\newcommand{\mssc}{\textsc{MSSC}}
\newcommand{\msscL}{\textsc{Min Sum Set Cover}}
\newcommand{\ol}[1]{\overline{#1}}

\date{} 

\title{Online Learning for Min Sum Set Cover and Pandora’s Box}
\author{
Evangelia Gergatsouli \\ UW-Madison \\ {\tt gergatsouli@wisc.edu} \and
Christos Tzamos \\ UW-Madison \\ {\tt tzamos@wisc.edu} \and
}

\begin{document}

\maketitle

Two central problems in Stochastic Optimization are \msscL {} and
\pb.  In \pb {}, we are presented with $n$ boxes, each
containing an unknown value and the goal is to open the boxes in some order to
minimize the sum of the search cost and the smallest value found. Given a
distribution of value vectors, we are asked to identify a near-optimal search
order. \msscL {} corresponds to the case where values are either $0$ or
infinity.

In this work, we study the case where the value vectors are not drawn from a
distribution but are presented to a learner in an online fashion. We present a
computationally efficient algorithm that is constant-competitive against the
cost of the optimal search order. We extend our results to a bandit setting
where only the values of the boxes opened are revealed to the learner after
every round. We also generalize our results to other commonly studied variants
of \pb {} and \msscL {} that involve selecting more than a
single value subject to a matroid constraint.

\setcounter{page}{0}
\thispagestyle{empty}
\newpage

\section{Introduction}\label{sec:intro}
One of the fundamental problems in stochastic optimization is \pb {} problem,
first introduced by Weitzman in \cite{Weit1979}. The problem asks to select
among $n$ alternatives, called boxes, one with a low value. In the stochastic
version of the problem, it is assumed that values in the boxes are drawn from a
known distribution and the actual realization of any box can be revealed at a
cost after inspection.

The goal is to design an algorithm that efficiently searches among the $n$
alternatives to find a low value while also paying a low inspection cost. We
thus aim to minimize the sum of the search costs the algorithm pays and the
value of the alternative(s) it chooses in the end.  In the standard version of
\pb {} a single value must be chosen, but we also consider common
generalizations that require $k$ distinct alternatives to be chosen, or
alternatives that form a matroid basis.

While most of the literature has focused on the stochastic case, where there is
a known distribution of values given in advance, we instead consider an \emph{online}
version of the problem played over $T$ rounds, where in each round a different
realization of values in the boxes is adversarially chosen. The goal of the
learner is to pick a good strategy of opening boxes in every round that
guarantees low regret compared to choosing in hindsight the optimal policy for
the $T$ rounds from a restricted family of policies. 

In this work, we mainly consider policies that fix the order in which boxes are
explored but have free arbitrary stopping rules. Such policies are called
\emph{partially-adaptive} and are known to be optimal in many cases, most
notably in the stochastic version where the values of the boxes are drawn
independently. Such policies are also optimal in the special case of the \pb {}
problem where the values in the boxes are either $0$ or $\infty$.  This case is
known as the \msscL {} problem (\mssc) and is a commonly studied
problem in the area of approximation algorithms.

\subsection{Our Results}
Our work presents a simple but powerful framework for designing online learning
algorithms for \pb {}, \mssc {} and other related problems. Our
framework yields approximately low-regret algorithms for these problems through
a three step process:
\begin{enumerate}
	\item We first obtain convex relaxations of the instances of every round.
	\item We then apply online-convex optimization to obtain good fractional
			solutions to the relaxed instances achieving low regret.
	\item We finally round the fractional solutions to integral solutions for
			the original instances at a small multiplicative loss.
\end{enumerate}

Through this framework, we obtain a
\begin{itemize}
	\item $9.22$-approximate no-regret algorithm for the problem of selecting $1$ box.
	\item $O(1)$-approximate no-regret algorithm for the problem of selecting $k$ boxes.
	\item $O(\log k)$-approximate no-regret algorithm for the problem of selecting a rank $k$ matroid basis.
\end{itemize}

We start by presenting these results in the \textbf{full information} setting (section \ref{sec:full_info})
where the values of all boxes are revealed after each round, 
once the algorithm has made its choices. 

A key contribution of our work is to further extend these results to a
more-realistic \textbf{bandit setting} (section \ref{sec:bandit_info}). In this setting, the algorithm only
observes the values for the boxes it explored in each round and can only use
this information to update its strategy for future rounds. In each round there
is also the option of obtaining the full information by paying a price.  We
show that even under this more pessimistic setting we can obtain approximately
no-regret algorithm with the same approximation guarantees as above.

We also provide stronger regret guarantees against more restricted classes of
algorithms for the \pb {} and \mssc {} problems that are non-adaptive (section~\ref{sec:na}).

All the algorithms we develop in this paper are computationally efficient.  As
such, the approximation guarantees given above are approximately tight since it
is NP-hard to improve on these beyond small constants even when competing with
the simpler non-adaptive benchmark. In particular, it was shown in~\cite{FeigLovaTeta2004}
that even the special case of \mssc {} is APX-hard
and cannot be approximated within a smaller factor than $4$.  It is an
interesting open question to what extent these bounds can be improved with
unlimited computational power.  While in the stochastic version, this would
trivialize the problem, in the online setting the obtained approximation
factors may still be necessary information theoretically.

\subsection{Comparison with Previous Work}
Our work is closely related to the work of \cite{ChawGergTengTzamZhan2020}. In
that work, the authors study a stochastic version of \pb {} with an arbitrarily
correlated distribution and aim to approximate the optimal partially adaptive
strategies.  We directly extend all the results
of~\cite{ChawGergTengTzamZhan2020} in the online non-stochastic setting, where
we are required at each round to solve an instance of the problem. 

Another very related paper is the work of \cite{FotaLianPiliSkou2020} that also
studies the online learning problem but focuses specifically on the \msscL {}
problem and its generalization (GMSSC) that asks to select $k$ alternatives
instead of one. Our work significantly improves their results in several ways.
\begin{itemize}
	\item We provide a simpler algorithm based on online convex optimization
			that does not rely on calculating gradients. We immediately obtain
			all our results through the powerful framework that we develop.
	\item This allows us to study more complex constraints like matroid rank
			constraints as well as study the more general \pb. It is
			challenging to extend the results of \cite{FotaLianPiliSkou2020} to
			such settings while keeping the required gradient computation task
			computationally tractable.
	\item Finally, we extend their results to a more natural bandit setting,
			where after each round we only have information about the
			alternatives that we explored rather than the whole instance.
\end{itemize}

In another recent work similar to ours, Esfandiari et al.
\cite{EsfaHajiLuciMitz2019} consider a Multi-armed bandit version of \pb{}
problem which however greatly differs with ours in the following ways.
\begin{itemize}
	\item In their setting each box has a type, and the algorithm is required to pick one box \textbf{per type}, 
				while in our case the game is independent in each round.
	\item Their benchmark is a ``prophet" who can choose the maximum reward per
			type of box, at the end of $T$ rounds.
	\item The decision to pick  a box is irrevocable\footnote{The algorithm
			decides when seeing a box whether to select it or not, and cannot
	``go back" and select the maximum value seen.} and they only consider
	threshold policies, as they relate the problem to prophet inequalities (see
	surveys \cite{HillKert1992,Luci2017,CorrFoncHoekOostVred2018} for more
	details on prophet inequalities).
\end{itemize}
\subsection{Related Work}
We model our search problem using \pb {}, which was first introduced by Weitzman in the Economics literature
\cite{Weit1979}. Since then, there has been a long line of research studying
\pb {} and its variants e.g. where boxes can be selected without inspection
\cite{Dova2018,BeyhKlein2019}, there is correlation between the
boxes~\cite{ChawGergTengTzamZhan2020}, the boxes have to be inspected in a
specific order~\cite{BoodFuscLazoLeon2020} or boxes are inspected in an  online
manner~\cite{EsfaHajiLuciMitz2019}. Some work is also done in the 
generalized setting where more information can be obtained for a
price~\cite{CharFagiGuruKleiRaghSaha2002,GuptKuma2001,
ChenJavdKarbBagnSrinKrau2015, ChenHassKarbKrau2015}. Finally a long line of research considers 
more complex combinatorial constraints like budget constraints
\cite{GoelGuhaMuna2006}, packing constraints \cite{GuptNaga2013}, matroid
constraints \cite{AdamSvirWard2016}, maximizing a submodular function
\cite{GuptNagaSing2016, GuptNagaSing2017}, an approach via Markov chains
\cite{GuptJianSing2019} and various packing and covering constraints for both
minimization and maximization problems~\cite{Sing2018}. In the special case of
MSSC, the line of work was initiated by \cite{FeigLovaTeta2004}, and continued with
improvements and generalizations to more complex constraints~\cite{AzarGamzYin2009,
MunaBabuMotwWido2005, BansGuptRavi2010, SkutWill2011}.

On the other hand, our work advances a recent line of research on the
foundations of data-driven algorithm design, started by Gupta and
Roughgarden~\cite{GuptRoug2016}, and continued by~\cite{BalcNagaViteWhit2016,
		BalcDickSandVite2018, BalcDickVite2018, KleiLeytLuci2017,
WeisGyorSzep2018,AlabKalaLigeMuscTzamVite2019}, where they study parameterized
families of algorithms in order to learn parameters to optimize the expected
runtime or performance of the algorithm with respect to the underlying
distribution. Similar work was done before \cite{GuptRoug2016} on
self-improving algorithms~\cite{AiloChazClarLiuMulzSesh2006, ClarMulzSesh2012}.

Furthermore, our results directly simplify and generalize the results of
\cite{FotaLianPiliSkou2020} in the case of partial feedback. Related to the
partial feedback setting, \cite{FlaxKalaMcma2005} consider single bandit
feedback and \cite{AgraDekeXiao2010} consider multi-point bandit feedback. Both
these works focus on finding good estimators for the gradient in order to run
a gradient descent-like algorithm. For more pointers to the online convex
optimization literature, we refer the reader to the survey by
Shalev-Shwartz~\cite{Shwa2012} and the initial primal-dual analysis of the
\emph{Follow the Regularized Leader} family of algorithms by
\cite{ShwaSing2007}.

\section{Preliminaries}\label{sec:prelims}
We evaluate the performance of our algorithms using \emph{average regret}. We
define the \textbf{average regret} of an algorithm $\mathcal{A}$ against a
benchmark $\opt$, over a time horizon $T$ as
\begin{equation}
		\text{Regret}_\opt(\mathcal{A}, T) = \frac{1}{T}\sum_{t=1}^T \lp( \mathcal{A}(t) - \opt(t) \rp)
\end{equation}\label{eq:regret}
where $\mathcal{A}(t)$ and $\opt(t)$ is the cost at round $t$ of $\mathcal{A}$
and $\opt$ respectively. We similarly define the average
\textbf{$\alpha$-approximate regret} against a
benchmark $\opt$ as
\begin{equation}
		\text{$\alpha$-Regret}_\opt(\mathcal{A}, T) = \frac{1}{T}\sum_{t=1}^T \lp( \mathcal{A}(t) - \alpha \opt(t) \rp)
.\end{equation}\label{eq:approx-regret}

We say that an algorithm $\mathcal{A}$ is \textbf{no regret} if
$\text{Regret}_\opt(\mathcal{A}, T) = o(1)$. Similarly, we say that
$\mathcal{A}$ is \textbf{$\alpha$-approximate no regret} if
$\text{$\alpha$-Regret}_\opt(\mathcal{A}, T) = o(1)$. Observe the we are always
competing with an oblivious adversary, that selects the one option that
minimizes the total loss over all rounds.

\subsection{Problem Definitions}

In \pb {} we are given a set $\boxes$ of $n$ boxes with unknown costs and a set of possible scenarios
that determine these costs. In each round $t\in[T]$, an adversary chooses the
instantiation of the costs in the boxes, called a \emph{scenario}. Formally, a
scenario at time $t$ is a vector $\bm{c}(t) \in \mathbb{R}^n$ for any $t\in
[T]$, where $c_{i}^s$ denotes the cost for box $i$ when scenario $s$ is
instantiated. Note that without loss of generality, we can assume that $c_i\leq
n$, since if some is more than $n$ we can ignore them, and if all are above $n$
we automatically get a $2$ approximation\footnote{Since opening all boxes to
find the minimum value costs us at most $n+\min_{i\in \boxes} c_i$, and the optimal also pays at
least $n$}.

The goal of the algorithm at every round is to choose a box of small cost
while spending as little time as possible gathering information. The algorithm
cannot directly observe the instantiated scenario, however, it is allowed to
``open" boxes one at a time. When opening a box, the algorithm observes the
cost inside the box. In total, we want to minimize the regret over $T$ rounds, relative to the optimal algorithm.

Formally, let $\mathcal{P}_t$ and $c_{i}^t$ be the set of boxes opened and the
cost of the box selected respectively by the algorithm at round $t\in [T]$. The
cost of the algorithm $\mathcal{A}$ at round $t$ is $\mathcal{A}(t) =  \min_{i\in \mathcal{P}_t} c_i^t +
|\mathcal{P}_t| $ and the goal is to minimize regret $ \text{Regret}_\opt(\mathcal{A},T)$.

Any algorithm can be described by a pair $(\sigma, \tau)$, where
$\sigma$ is a permutation of the boxes representing the order in which they are opened, 
and $\tau$ is a stopping rule -- the time at which the algorithm stops
opening and returns the minimum cost it has seen so far. Observe that in its
full generality, an algorithm may choose the next box to open and the stopping time 
as a function of the identities and costs of the previous opened boxes. 

\paragraph{Different Benchmarks.} As observed
in~\cite{ChawGergTengTzamZhan2020}, optimizing over the class of all such
algorithms is intractable, therefore simpler benchmarks are considered.

\begin{itemize}
		\item \textbf{The Non-adaptive Benchmark (NA)}: in this case the
				adversary chooses all the $T$ scenarios about to come, and
				selects a fixed set of boxes to open, which is the same in
				every round. In this case, the $\opt(t)$ term in the regret does not depend
				on $t$.
		\item \textbf{The Partially-adaptive Benchmark (PA)}: in this case, the
				adversary can have a different set of boxes to open in each
				round, which can depend on the algorithm's choices in rounds
				$1,\ldots, t-1$.  
\end{itemize}
\paragraph{An important special case.} A special case of \pb {} is the \msscL
{} problem (\mssc).  In this problem, the costs inside the boxes are either $0$
or $\infty$. We say a scenario is \emph{covered} or \emph{satisfied} if, in our
solution, we have opened a box that has value $0$ for this scenario.

\paragraph{General feasibility constraints.} We also study two more complex
extensions of the problem. In the first one  we are required to select exactly
$k$ boxes for some $k\ge 1$, and in the second, the algorithm is required to
select a basis of a given matroid. We design \textbf{partially-adaptive}
strategies that are approximately no-regret, for the different constraints and
benchmarks described in this section.

\subsection{Relaxations}
\subsubsection{Scenario - aware Relaxation}

 Observe that the class of partially-adaptive strategies is still too large and
 complex, since the stopping rule can arbitrarily depend on the costs observed
 in boxes upon opening. One of the main contributions of
 \cite{ChawGergTengTzamZhan2020}, which we are using in this work too, is that they showed
 it is enough to design a strategy that chooses an ordering of the boxes and 
 performs well, assuming that we know when to stop. This relaxation of partially-adaptive,
 called \emph{scenario-aware partially-adaptive} (SPA), greatly diminishes the
 space of strategies to consider, and makes it possible to design competitive
 algorithms, at the cost of an extra constant factor. This is formally stated
 in the lemma below. The proof can be found
 in~\cite{ChawGergTengTzamZhan2020} and it is based on a generalization of
 ski-rental~\cite{KarlManaMcgeOwic1990}.

 \begin{lemma}[Simplification of Theorem 3.4 from~\cite{ChawGergTengTzamZhan2020}]\label{thm:ski-rental}
	 For a polynomial, in the number of boxes, $\alpha$-approximate algorithm
		 for scenario-aware partially adaptive strategies, there exists a
		 polynomial time algorithm that is a
		 $\frac{e}{e-1}\alpha$-approximation partially-adaptive strategy.
 \end{lemma}

 \subsubsection{Fractional Relaxation and Rounding}\label{subsec:relaxation}
This first relaxation allows us to only focus on designing efficient SPA
strategies which only require optimizing over the permutation of boxes. However
both \mssc {} and \pb {}  are non-convex problems. We tackle this issue by
using a convex relaxation of the problems, given by their linear programming
formulation.

\begin{definition}[Convex Relaxation]
	Let $\Pi$ be a minimization problem over a domain $X$ with $g:
		X\rightarrow \mathbb{R}$ as its objective function, we say that a function $\overline{g}: \overline{X}
		\rightarrow \mathbb{R}$ is a convex relaxation of $g$, if 
		\begin{enumerate}
			\item The function $\overline{g}$ and its domain $\overline{X}$ are convex.
			\item $X \subseteq \overline{X}$ and for any $x \in X$, $\overline{g}(x) \le g(x)$.
		\end{enumerate}
\end{definition}


Using this definition, for our partially-adaptive benchmark we relax the domain $\overline{X} = \{x \in
[0,1]^{n\times n}: \sum_i x_{it} = 1 \text{ and } \sum_t x_{it} = 1\}$ to be
the set of doubly stochastic $n\times n$ matrices. We use a convex 
relaxation $\overline{g}^s$ similar to the one from the generalized min-sum set
cover problem in \cite{BansGuptRavi2010} and \cite{SkutWill2011}, but \emph{scenario dependent}; for a given scenario $s$, the relaxation $\ol{g}^s$ changes.
We denote by
$\timeset$ the set of $n$ time steps, by $x_{it}$ the indicator variable for
whether box $i$ is opened at time $t$, and by $z_{it}^s$ the indicator of
whether box $i$ is selected for scenario $s$ at time $t$. We define the relaxation
$\overline{g}^s(\bm{x})$ as 
	\begin{align*}
		\text{min}_{z \ge 0}  \quad & \sum_{i\in\boxes, 
		 t\in \timeset} (t +c_{i}^s) z_{it}^s  &  
		  \tag{Relaxation-SPA} \label{lp-spa}\\
		  \text{s.t.}\quad & \sum_{t\in\timeset, i\in \boxes}z_{it}^s = 1, & \\
		&  \hspace{0.7cm}  z_{it}^s \leq x_{it}, & i\in \boxes,t\in\timeset.
	\end{align*}
Similarly, we also relax the problem when we are required to pick $k$ boxes
(\ref{lp-pa-k}) and when we are required to pick a
matroid basis (\ref{lp-pa-matroid}).  


Leveraging the results of \cite{ChawGergTengTzamZhan2020}, in
sections~\ref{apn:pa_1}, \ref{apn:PA_k} and \ref{apn:PA_matroid} of the appendix, we
show how to use a rounding that does not depend on the scenario chosen in order
to get an approximately optimal integer solution, given one for the relaxation.
Specifically, we define the notion of $\alpha$-approximate rounding.

\begin{definition}[$\alpha$-approximate rounding]
	Let $\Pi$ be a minimization problem over a domain $X$ with $f: X\rightarrow
	\mathbb{R}$ as its objective function and a convex relaxation
	$\overline{f}: \overline{X} \rightarrow \mathbb{R}$. Let $\overline{x} \in
	\overline{X}$ be a solution to $\Pi$ with cost
	$\overline{f}(\overline{x})$. Then an $\alpha$-approximate rounding is a
	an algorithm that given $\overline{x}$ produces a solution $x\in X$ with cost
	\[ 
		f(x) \leq \alpha \overline{f}(\overline{x})
	\]
\end{definition}

\section{Full Information Setting}\label{sec:full_info}
We begin by presenting a general technique for approaching \pb {} type of
problems via Online Convex Optimization (OCO). 
Initially we observe, in the following theorem, that we can combine
\begin{enumerate}
		\item a rounding algorithm with good approximation guarantees,
		\item an online minimization algorithm with good regret guarantees
\end{enumerate}
to obtain an algorithm with good regret guarantee.

\begin{theorem}\label{thm:full_info}
		Let $\Pi$ be a minimization problem over a domain $X$ and
		$\overline{\Pi}$ be the convex relaxation of $\Pi$ over convex domain
		$\overline{X} \supseteq X$.

		If there exists an $\alpha$-approximate rounding algorithm $\mathcal{A}: \overline{X} \rightarrow X$ 
		for any feasible solution $\overline{x}\in \overline{X}$ to a feasible solution $x \in X$ then, 
		any online minimization algorithm for $\overline{\Pi}$ that achieves 
		regret $\text{Regret}(T)$ against a benchmark $\opt$, gives 
		$\alpha$-approximate regret $\alpha \text{Regret}(T)$ for $\Pi$.
\end{theorem}

\begin{proof}[Proof of Theorem~\ref{thm:full_info}]
		Let $f_1,...,f_T$ be the online sequence of functions presented in
		problem $\Pi$, in each round $t\in [T]$, and let
		$\overline{f}_1,...,\overline{f}_T$ be their convex relaxations in
		$\overline{\Pi}$. 
	
		Let $\overline{x}_t\in \overline{X}$ be the solution the online convex
		optimization algorithm gives at each round $t\in [T]$ for problem $\overline{\Pi}$.
		Calculating the total expected cost of $\Pi$, for all time steps $t\in[T]$ we have that
\begin{align*}
		\E{	\sum_{t=1}^T f_t\big( \mathcal{A}(\overline{x_t})\big) } & \leq \alpha \sum_{t=1}^T \overline{f}_{t}(\overline{x}_t) \\
		& \leq \alpha \lp( \text{Regret}(T) + \min_{x\in \overline{X}} \sum_{t=1}^T \overline{f}_{t}(x)\rp) \\
		& \leq \alpha \lp( \text{Regret}(T) + \min_{x\in X} \sum_{t=1}^T f_{t}(x)\rp).
\end{align*}

By rearranging the terms, we get the theorem.
\end{proof}

Given this theorem, in the following sections we show (1) how to design an algorithm with 
a low regret guarantee for \pb{} (Theorem~\ref{lem:regr_fractional}) and (2) how
to obtain rounding algorithms with good approximation guarantees, using
the results of~\cite{ChawGergTengTzamZhan2020}.

\subsection{Applications to \pb {}  and \mssc {}}
Applying Theorem~\ref{thm:full_info} to our problems, in their initial
non-convex form, we are required to pick an integer permutation of boxes. The
relaxations, for the different benchmarks and constraints, are shown in 
\ref{lp-spa}, \ref{lp-pa-k} and \ref{lp-pa-matroid}.

We denote by $\ol{g}^s(\bm{x})$ the objective function of the scenario aware relaxation
of the setting we are trying to solve e.g for selecting $1$ box we have
\ref{lp-spa}. Denote by $\overline{X} = [0,1]^{n\times n}$ the solution space.
We can view this problem as an online convex optimization one as follows.
\begin{enumerate}
		\item At every time step $t$ we pick a vector $\bm{x}_t \in \overline{X}$, where $\overline{X}$ is a convex set.
	\item The adversary picks a scenario $s\in \scenarios$ and therefore a
			function $f^{s}: \overline{X} \rightarrow \mathbb{R}$ where $f^s = \ol{g}^s$ and we incur loss
				$f^{s}(\bm{x}_t) = \ol{g}^s(\bm{x}_t)$. 
				Note that $f^{s}$ is convex in all cases (\ref{lp-spa}, \ref{lp-pa-k}, \ref{lp-pa-matroid}).
		\item We observe the function $f^{s}$ for all points $\bm{x}\in \overline{X}$. 
\end{enumerate}

A family of algorithms that can be applied to solve this problem is called
\emph{Follow The Regularized Leader (FTRL)}. These algorithms work by picking,
at every step, the solution that would have performed best so far while also
adding a regularization term for stability. For the \emph{FTRL} family of algorithms 
we have the following guarantees.

\begin{theorem}[Theorem 2.11 from \cite{Shwa2012}]\label{thm:FTRL}
	Let $f_1,\ldots ,f_T$ be a sequence of convex functions such that each
$f_t$ is $L$-Lipschitz with respect to some norm.  Assume that FTRL is run on
the sequence with a regularization function $U$ which is $\eta$-strongly-convex
with respect to the same norm. Then, for all $u\in C$ we have that
		$\text{Regret}(\text{FTRL}, T) \cdot T \leq U_{\max} - U_{\min} + TL^2 \eta$
\end{theorem}

Our algorithm works similarly to \emph{FTRL}, while additionally rounding the fractional
solution, in each step, to an integer one.  The algorithm is formally described
in Algorithm~\ref{algo:full_info}, and we show how to choose the regularizer
$U(\bm{x})$ in Theorem~\ref{lem:regr_fractional}. \\

\begin{algorithm}[H]
	\KwIn{$\Pi = (\mathcal{F}, \opt):$ the problem to solve,
 	$\mathcal{A}_\Pi:$ the rounding algorithm for $\Pi$}
	Denote by $f^s(\bm{x}) = $ fractional objective function\\
	Select regularizer $U(\bm{x}) $ according to Theorem~\ref{lem:regr_fractional}\\
	$\overline{X} = $ space of fractional solutions\\
	\For {Each round $t\in [T]$}{
			Set $\bm{x}_t = \min_{\bm{x}\in \overline{X}} \sum_{\tau=1}^{t-1} f^{s_\tau}(\bm{x}) + U(\bm{x}) $\\
	Round $\bm{x}_t$ to $\bm{x}_t^{\text{int}}$ according to $\mathcal{A}_\Pi$\\
	Receive loss $f^s( \bm{x}_t^{\text{int}})$
	}
\caption{Algorithm $\mathcal{A}$ for the full information case.}\label{algo:full_info}
\end{algorithm}
\mbox{}\\
We show the guarantees of our algorithm above using Theorem~\ref{thm:FTRL}
which provides regret guarantees for \emph{FTRL}. The proof of
Theorem~\ref{lem:regr_fractional} is deferred to section~\ref{apn:full_info} of
the appendix.
\begin{restatable}{theorem}{regrFractional}\label{lem:regr_fractional}
	The average regret of Algorithm~\ref{algo:full_info} is
		\[
		\text{Regret}_{PA}(\mathcal{A}, T) \leq 2n \sqrt{\frac{\log n}{T}}
		\]
		achieved by setting $U(\bm{x}) = \lp( \sum_{i=1}^n \sum_{t=1}^n x_{it} \log
		x_{it} \rp) /\eta$ as the regularization function, and $\eta= \sqrt{\frac{\log n}{T}}$.
\end{restatable}

Finally, using Theorem~\ref{thm:full_info} we get
Corollary~\ref{cor:full_info_pa} for competing with the partially-adaptive
benchmark for all different feasibility constraints (choose $1$, choose $k$ or
choose a matroid basis), where we use the results of Corollary~\ref{cor:a_rounding_pa}, to obtain the guarantees for the rounding algorithms.
\begin{corollary}[Competing against PA, full information]\label{cor:full_info_pa}
	In the full information setting, Algorithm~\ref{algo:full_info} is 
	\begin{itemize}	
			\item  $9.22$-approximate no regret for \textbf{choosing $1$ box} 	
			\item $O(1)$-approximate no regret for \textbf{choosing $k$ boxes} 
			\item $O(\log k)$-approximate no regret for \textbf{choosing a matroid basis} 
	\end{itemize}
\end{corollary}

\begin{remark}
	In the special case of MSSC, our approach obtains the tight
		$4$-approximation of the offline
		case~\cite{FeigLovaTeta2004}. The details of this are
		deferred to section~\ref{apn:pa_1} of the Appendix. This
		result improves on the previous work~\cite{FotaLianPiliSkou2020} who obtain a $11.713$-approximation.
\end{remark}

\section{Bandit Setting}\label{sec:bandit_info}
Moving on to a bandit setting for our problem, where we do not observe the
whole function after each step.  Specifically, after choosing $\bm{x}_t\in
\overline{X}$ in each round $t$, we only observe a loss $f^s (\bm{x}_t)$ at
the point $\bm{x}_t$ we chose to play and not for every $\bm{x}\in \overline{X}$.
This difference prevents us from directly using any online convex optimization
algorithm, as in the full information setting of section~\ref{sec:full_info}.
However, observe that if we decide to open all $n$ boxes, this is equivalent to
observing the function $f^s$ for all $\bm{x}\in \overline{X}$, since we learn the
cost of all permutations. 

We exploit this similarity by randomizing between running \emph{FTRL} and
paying $n$ to open all boxes. Specifically we split $[T]$ into $T/k$ intervals
and choose a time, uniformly at random in each one, when we are going to open
all boxes $n$ and thus observe the function on all inputs. This process is formally described in Algorithm~\ref{algo:base_PA}, 
and we show the following guarantees.

\begin{theorem}\label{thm:regret_pa}
	The average regret for Algorithm~\ref{algo:base_PA}, for $k=\lp(
		\frac{n}{2L+\sqrt{\log n}} \rp)^{2/3}T^{1/3}$ and loss functions that are $L$-Lipschitz is 
	\[
			\E{\text{Regret}_{PA}(\mathcal{A}_{\text{PA}},T)} \leq 2\lp(2L\log n + n\rp)^{2/3} \cdot n^{1/3} \cdot T^{-1/3}
	.\] 
\end{theorem}

 \begin{algorithm}[H]
		 \caption{$\mathcal{A}_{PA}$ minimizing regret against PA}
 	\label{algo:base_PA}
	Get parameter $k$ from Theorem~\ref{thm:regret_pa}\\
		 Select regularizer $U(\bm{x})$ according to Theorem~\ref{thm:regret_pa}\\
	 Split the times $[T]$ into $T/k$ intervals $\mathcal{I}_1\ldots, \mathcal{I}_{T/k}$ \\
	 $\mathcal{R} \leftarrow \emptyset$  \tcp{Random times for each $\mathcal{I}_i$}
	 \For{Every interval $\mathcal{I}_i$}{
		Pick a $t_p$ uniformly in $\mathcal{I}_i$\\
		\For{All times $t\in \mathcal{I}_i$}{
				\uIf{$t=t_p$}{
		 $\mathcal{R} \leftarrow \mathcal{R} \cup \{t_p\}$\\
			 Open all boxes\\
			 Get feedback $f^{s_{t_p}}$\\
		 }
		 \Else{
			$\bm{x}_t \leftarrow \text{argmin}_{\bm{x}\in \overline{X}} 
					\sum_{\tau\in \mathcal{R}} f^{s_\tau}(\bm{x}) + U(\bm{x})$
			}
			 }
	 }
 \end{algorithm}
 \mbox{}\\ 
 To analyze the regret of Algorithm~\ref{algo:base_PA} and prove Theorem~\ref{thm:regret_pa}, 
 we consider the regret of two related settings. 
\begin{enumerate}
		\item  In the first setting, we consider a full-information 
 online learner that observes at each round $t$ a single function sampled uniformly among the $k$ functions
 of the corresponding interval  $\mathcal{I}_t$. We call this setting \emph{random costs}.
				\item  In the second setting, we again consider a full-information 
 online learner that observes at each round $t$ a single function which is the average of the $k$ functions
 in the corresponding interval  $\mathcal{I}_t$. We call this setting \emph{average costs}.
\end{enumerate}
 The following lemma, shows that any online algorithm for the random cost setting yields low regret even for 
 the average costs setting.

 \begin{lemma}\label{lem:avg_rand}
	 Any online strategy for the \emph{random costs} setting with expected average regret $R(T)$ gives 
	 expected average regret at most $R(T) + n/\sqrt{kT}$ 
	 for the equivalent \emph{average costs setting}.
 \end{lemma}

 \begin{proof}[Proof of Lemma~\ref{lem:avg_rand}]
	Denote by $\overline{f}_t = \frac 1 k \sum_{i=1}^k f_{t_i}$ the cost
	function corresponding to the average costs setting and by $f^r_t =
	f_{t_i}$ where $i\sim U\lp([k]\rp)$ the corresponding cost function for the
	random costs setting.
	Let  $x^* = \text{argmin}_{\bm{x}\in \overline{X}} \sum_{t=1}^{T/k}
	\overline{f}_t(\bm{x})$ be the minimizer of the $\overline{f}_t$ over the $T/k$ rounds.

	We also use $X_t = \overline{f}_t(\bm{x}_t) - f^r_t(\bm{x}_t)$, to
	denote the difference in costs between the two settings for each interval
		 (where $\bm{x}_t$ is the action taken at each interval $t$ by the
		 random costs strategy). Observe that this is a random variable
		 depending on the random choice of time in each interval.
	We have that
	\begin{align*}
			\E{\sum_{t=1}^{T/k} |X_t|} & \leq \lp( \E{\lp( \sum_{t=1}^{T/k} X_t\rp) ^2}\rp)^{1/2}\\
			& = \lp( \E{ \sum_{t=1}^{T/k} X_t^2}\rp)^{1/2} \\
			&\leq n\sqrt{\frac{T}{k}}. 
	\end{align*}
	The two inequalities follow by Jensen's inequality and the fact that
		 $X_t$'s are bounded by $n$. The equality is because the random
		 variables $X_t$ are martingales, i.e.  $\E{X_t | X_1, ..., X_{t-1}} =
		 0$, as the choice of the function at time $t$ is independent of the
		 chosen point $\bm{x}_t$.

	We now look at the average regret of the strategy $\bm{x}_t$ for the average cost setting. We have that
	\begin{align*}
			\frac{1}{T} \E{ \sum_t \overline{f}_t(\bm{x}_t)  } - R(T) - \frac{n}{\sqrt{kT}} & 
			\leq \frac{1}{T} \E{ \sum_t f^r_t(\bm{x}_t)  } - R(T)\\
					&\leq \frac{1}{T}\E { \min_{\bm{x}} \sum_t f^r_t(\bm{x}) }\\
							&\leq  \frac 1 T \E{ \sum_t f^r_t(\bm{x}^*) } \\
					&= \sum_t \overline{f}_t(\bm{x}^*) 
\end{align*}
	which implies the required regret bound.

\end{proof}

Given this lemma, we are now ready to show Theorem~\ref{thm:regret_pa}.

\begin{proof}[Proof of Theorem~\ref{thm:regret_pa}]
	To establish the result, we note that the regret of our algorithm is equal to the regret achievable in the average cost setting
	multiplied by $k$ plus $n T/k$ since we pay $n$ for opening all boxes once in each of the $T/k$ intervals.
	Using Lemma~\ref{lem:avg_rand}, it suffices to bound the regret in the random costs setting.
		Let $U(\bm{x}): [0,1]^{n\times n} \rightarrow \mathbb{R}$ be an $\eta/n$-strongly
		convex regularizer used in the \emph{FTRL} algorithm. We are using
		$U(\bm{x}) = \lp( \sum_{i=1}^n \sum_{t=1}^n x_{it} \log
			x_{it} \rp) /\eta $, which is $\eta/n$-strongly convex from
			Lemma~\ref{lem:strong_conv} and is at most $(n\log n)/\eta$ as we observed in corollary~\ref{cor:full_info_pa}.
		Then from Theorem~\ref{thm:FTRL}, we get that the average regret for the corresponding random costs setting is
		$2 L \sqrt{\frac{ \log n }{kT}}$.

		Using Lemma~\ref{lem:avg_rand}, we get that the total average regret $R(T)$ of our algorithm is
		\[ 
		R(T) \leq k \cdot 2 L \sqrt{\frac{\log n }{kT}} + k \cdot n/\sqrt{kT} + \frac n k.
		\]
		Setting
		$k= \lp( \frac{n}{2L\log n+n} \rp)^{2/3}T^{1/3}$ the theorem follows.
\end{proof}

Finally, using Theorem~\ref{thm:regret_pa} we can get the same guarantees as
the full-information setting, using the $\alpha$-approximate rounding for each case (Corollary~\ref{cor:a_rounding_pa}).
\begin{corollary}[Competing against PA, bandit]\label{cor:bandit_info_pa}
	In the bandit setting, Algorithm~\ref{algo:base_PA} is 
	\begin{itemize}	
			\item  $9.22$-approximate no regret for \textbf{choosing $1$ box} 
			\item $O(1)$-approximate no regret for \textbf{choosing $k$ boxes} 
			\item $O(\log k)$-approximate no regret for \textbf{choosing a matroid basis} 
	\end{itemize}
\end{corollary}

\section{Competing with the Non-adaptive}\label{sec:na}
We switch gears towards a different benchmark, that of the non-adaptive strategies. Similarly to the partially adaptive benchmark, here we we first present the linear programming for the non-adaptive benchmark as a function $\overline{f} :[0,1]^n \rightarrow \mathbb{R}$ with $\overline{f}(\bm{x})$ equal to 
\begin{align*}
\text{min}_{z \ge 0}  \quad & \sum_{i\in \boxes} x_i  +  \frac{1}{|\scenario|}\sum_{i\in \boxes,s\in \scenario}c_{i}^s z_{i}^s
  & \tag{LP-NA} \label{lp-na}\\
  \text{s.t.}\quad & \sum_{i\in \boxes}z_i^s = 1, & \forall s\in \scenario\\
	& \hspace{0.7cm} z_i^s  \leq x_i & \forall i\in \boxes, s\in \scenario 
\end{align*}
where $x_i$ is an indicator variable
for whether box $i$ is opened and $z_{i}^s$ indicates whether box $i$
is assigned to scenario $s$.

Note that the algorithms we provided for the partially-adaptive case cannot
be directly applied since the objective functions of \ref{lp-na}, \ref{lp-na-k}
and \ref{lp-na-matroid} are not $n$-Lipschitz. 
To achieve good regret bounds in this case, we design an algorithm that randomizes over an ``explore" and
an ``exploit" step, similarly to~\cite{AlabKalaLigeMuscTzamVite2019}, while
remembering the LP structure of the problem given constraints $\mathcal{F}$.
Observe that there is a ``global" optimal linear program (which is either
\ref{lp-na}, \ref{lp-na-k} or \ref{lp-na-matroid} depending on the constraints $\mathcal{F}$)
defined over all rounds $T$. Getting a new instance in each round is equivalent
to receiving a new (hidden) set of constraints.
We first describe two functions utilised by the algorithm in order to find a feasible 
fractional solution to the LP and to round it.

\begin{enumerate}
		\item \emph{Ellipsoid}($k, \mathcal{LP}$): finds and returns a feasible solution
		 to $\mathcal{LP}$ of cost at most $k$. By starting from a low $k$
			value and doubling at each step, lines $10$-$13$ result in us finding
			a fractional solution within $2$ every time.
	\item \emph{Round}($S_{t}, \mathcal{F}$): rounds the fractional feasible
			solution $S_t$ using the algorithm corresponding to $\mathcal{F}$.
			The rounding algorithms are presented in section~\ref{apn:NA} of
			the appendix. For selecting $1$ box we have
			Algorithm~\ref{algo:na-1}, for selecting $k$ boxes
			Algorithm~\ref{algo:na-k} and for selecting a matroid basis
			Algorithm~\ref{algo:na-matroid}.
 \end{enumerate}

The algorithm works in two steps; in the ``explore" step (line~7) opening all
boxes results in us exactly learning the hidden constraint of the current
round, by paying $n$. The ``exploit" step uses the information learned from the
previous rounds to open boxes and choose one.\\

 \begin{algorithm}[tb]
		 \caption{Algorithm $\mathcal{A}_\mathcal{F}$ for minimizing regret vs NA}
 	\label{algo:base_NA}
			\textbf{Input}: set of constraints $\mathcal{F}$\\
				  $\mathcal{LP} \leftarrow $ \ref{lp-na} or \ref{lp-na-k} or \ref{lp-na-matroid} (according to $\mathcal{F}$)\\
				  $\mathcal{C}_1 \leftarrow \emptyset$ \tcp{Constraints of LP}
			 \For {round $t\in [T]$} {
			draw $c\in U[0,1]$\\
			 \uIf{$c>p_t$} { 
				  Open all $n$ boxes, inducing new constraint $c_{new}$ \label{algo:step_if}\\
				 $\mathcal{C}_{t+1} \leftarrow \mathcal{C}_t \cup \{ c_{new} \}$ \\
				 $k\leftarrow 1$\\
				 \Repeat{$(\bm{x}, \bm{z}) $ is feasible}{
					 $(\bm{x}, \bm{z}) \leftarrow $ \textbf{Ellipsoid} ($k$, $\mathcal{LP}$)\\
					 $k \leftarrow 2k$
						 }
				}
				\Else  {\label{algo:step_else}
					 $S_{t} \leftarrow S_{t-1}$\\
				 $\pi \leftarrow $\textbf{Round}($S_{t}, \mathcal{F}$)\label{algo:na_round}\\
				  Open boxes according to order $\pi$
				 }
				}
 \end{algorithm}

 Observe that the cost of Algorithm~\ref{algo:base_NA} comes from three
 different cases, depending on what the result of the flip of the coin $c$ is in each round.
 \begin{enumerate}
		 \item If $c > p_t$, we and pay $n$ for opening all boxes.
		 \item If $c< p_t$ and we pay cost proportional to the LP (we have a feasible solution).
		 \item If $c < p_t$ and we pay cost proportional to $n$ (we did not have a feasible solution).
 \end{enumerate}

 We bound term 3 using mistake bound, and then continue by bounding terms 1 and 2 to get the bound on total regret.

 \subsection{Bounding the mistakes}\label{subsec:MB_NA}
 	We start by formally defining what is \emph{mistake bound} of an algorithm.
\begin{definition}[Mistake Bound of Algorithm $\mathcal{A}$]
	Let $\mathcal{A}$ be an algorithm that solves problem $\Pi$ and runs in $t\in [T]$ rounds with input $x_t$
		in each one. Then we define $\mathcal{A}$'s mistake bound as
		\[
				\mathcal{M}(\mathcal{A},T ) = \E{\sum_{t=1}^T \ind{x_t \text{ not feasible for }\Pi}}
			\]
		where the expectation is taken over the algorithm's choices.
\end{definition}

The main contribution in this section is the following lemma, that bounds the number of mistakes. 

\begin{restatable}{lemma}{mblemma}\label{thm:MB_simple}
	Algorithm~\ref{algo:base_NA} has mistake bound
	\[ 
		\mathcal{M}(\mathcal{A}_\mathcal{F}, T) \leq  O(n^2\sqrt{T})
	.\]	
\end{restatable}

The mistake bound applies to all the different constraints $\mathcal{F}$ we
consider. To achieve this, we leverage the fact that the ellipsoid algorithm,
running on the optimal LP corresponding to the constraints $\mathcal{F}$, needs
polynomial in $n$ time to find a solution. The proof works by showing that
every time, with probability $p_t$, we make progress towards the solution, and
since the ellipsoid in total makes polynomial in $n$ steps we also cannot make
too many mistakes. The proof of Lemma~\ref{thm:MB_simple} is deferred to
section~\ref{apn:NA_regret_general} of the appendix.

	\subsection{Regret for different constraints}\label{subsec:regret_NA}
Moving on to show regret guarantees of Algorithm~\ref{algo:base_NA} for the
different types of constraints. We start off with the special case where we are
required to pick one box, but all the costs inside the boxes are either $0$ or
$\infty$, and then generalize this to arbitrary costs and more complex
constraints.

\begin{theorem}[Regret for $0/\infty$]\label{thm:regret}
		Algorithm~\ref{algo:base_NA}, with $p_t = 1/\sqrt{T}$ has the following average regret, when
		$\mathcal{F} = \{\text{Select 1 box}\}$ and $c_i\in
		\{0,\infty\}$.
	\[
			\E{\text{Regret}_{NA}(\mathcal{A}_\mathcal{F},T)} \leq   \opt + O\lp( \frac{  n^2 }{\sqrt{T} } \rp)
		.\] 
\end{theorem}

\begin{proof}[Proof of Theorem~\ref{thm:regret}]
	Denote by $M$ the mistake bound term, bounded above in
		Lemma~\ref{thm:MB_simple}. We calculate the total average regret
	\begin{align*}
			\E{\text{Regret}_{NA}(\mathcal{A}_\mathcal{F}, T)} & + \opt  = \frac{1}{T} \lp( M + \sum_{t=1}^{T} \E{|S_t|} \rp)& \\
			& = \frac{1}{T} \lp( M + \sum_{t=1}^T p_tn + (1-p_t)\E{|S_t|}\rp) & \\
			 & \leq \frac{1}{T} \lp( M  + \sum_{t=1}^T p_tn + 2\opt\rp) & \\
			& \leq  M  + 2\opt +  n \sum_{t=1}^T p_t & \\
			& \leq  2\opt + O\lp(\frac{n^2}{\sqrt{T}}\rp)
	\end{align*}
		where initially we summed up the total regret of
		Algorithm~\ref{algo:base_NA} where the first term is the mistake bound
		from Lemma~\ref{thm:MB_simple}. Then we used the fact that $\opt_t\leq
		\opt$ and the solution found by the ellipsoid is within $2$, and in the
		last line  we used Since $\sum_{t=1}^T p_t \leq \sqrt{T} $
		from~\cite{AlabKalaLigeMuscTzamVite2019}. Finally, subtracting $\opt$
		from both sides we get the theorem.
\end{proof}

Generalizing this to arbitrary values $c_i\in \mathbb{R}$, we show that when we
are given a $\beta$ approximation algorithm  we get the following guarantees,
depending on the approximation factor.

\begin{restatable}{theorem}{regretNA}\label{thm:regret_NA}

	If there exists a partially adaptive algorithm $\mathcal{A}_\mathcal{F}$ that is $\beta$-competitive
	against the non-adaptive optimal, for the problem with constraints $\mathcal{F}$, then
		Algorithm~\ref{algo:base_NA}, with $p_t = 1/\sqrt{T}$ has the following regret.
	\[
		\E{\text{Regret}_{NA}(\mathcal{A}_\mathcal{F}, T)} \leq 2\beta\opt + O\lp(\frac{n^2}{\sqrt{T}}\rp) 
	.\] 
\end{restatable}
The proof follows similarly to the $0/\infty$ case, and is deferred
to section~\ref{apn:NA_regret_general} of the appendix. Combining the different
guarantees against the non-adaptive benchmark with
Theorem~\ref{thm:regret_NA} we get the following corollary.

\begin{corollary}[Competing against NA, bandit setting]\label{cor:bandit_na}
	In the bandit setting, when competing with the non-adaptive benchmark, Algorithm~\ref{algo:base_NA} is 
	\begin{itemize}	
			\item  $3.16$-approximate no regret for \textbf{choosing $1$ box} 
					(using Theorem~\ref{thm:na_1})
				\item $12.64$-approximate no regret for \textbf{choosing $k$ boxes} (using Theorem~\ref{thm:PA_vs_NA_k})
			\item $O(\log k)$-approximate no regret for \textbf{choosing a matroid basis} (using
					Theorem~\ref{thm:PA_vs_NA_matroid})
	\end{itemize}
\end{corollary}

\bibliography{refPB}
\bibliographystyle{alpha}

\newpage
\appendix
\section{Proofs from section~\ref{sec:full_info}}\label{apn:full_info}

The following lemma shows the strong convexity of the regularizer used in our
\emph{FTRL} algorithms.
\begin{lemma}[Convexity of Regularizer]\label{lem:strong_conv}
	The following function is $1/n$-strongly convex with respect to the $\ell_1$-norm.
	\[ 
		U(\bm{x}) = \sum_{i=1}^n \sum_{t=1}^n x_{it} \log x_{it}
		\]
		for a doubly-stochastic matrix $x\in [0,1]^{n \times n}$
\end{lemma}

\begin{proof}
		Since $U(\bm{x})$ is twice continuously differentiable we calculate $\nabla^2 U(\bm{x})$, which is a
		$n \times n$ diagonal matrix since 
		\[ \frac{\vartheta U(\bm{x}) }{\vartheta x_{kt} \vartheta x_{ij}} =
			\begin{cases}
				1/x_{ij} & \text{If } i=k \text{ and } j=t\\
					0& \text{Else}
			\end{cases}
		\]
		We show that $\bm{z} \nabla^2 U(\bm{x}) \bm{z} \geq \norm{\bm{z}}{1}^2$ for all $\bm{x}\in
		\mathbb{R}^{n^2}$. We make the following mapping of the variables for
		each $x_{ij}$ we map it to $p_{k}$ where $k= (i-1)n + j$. We have that
		\begin{align*}
				\bm{z} \nabla^2 U(\bm{x}) \bm{z}& = \sum_{i=1}^{n^2} \frac{(z_i)^2}{p_i}  \\
					& = \frac{1}{n} \lp( \sum_{i=1}^{n^2} p_i \rp) \sum_{i=1}^{n^2} \frac{(z_i)^2}{p_i} \\
					& \geq \frac{1}{n}\lp(\sum_{i=1}^{n^2} \sqrt{p_i} \frac{|z_i|}{\sqrt{p_i}}\rp)^2\\
				& = \frac{1}{n}\norm{\bm{z}}{1}^2.
		\end{align*}
		where in the second line we used that $x_{ij}$'s are a doubly
		stochastic matrix, and then Cauchy-Schwartz inequality.
\end{proof}

\regrFractional*
\begin{proof}
		Initially observe that by setting $x_{ij}=1/n$ we get $U_{\max} - U_{\min} =  (n \log
		n)/\eta$, since we get the maximum entropy when the values are all equal. Additionally, 
		from Lemma~\ref{lem:strong_conv} we have that $U(\bm{x})$ is
		$\eta/n$-strongly convex. Observing also that the functions in all cases
		are $n$-Lipschitz and using Theorem~\ref{thm:FTRL} we obtain the
		guarantee of the theorem, by setting $\eta = \frac{\sqrt{\log
		n}}{\sqrt{T}}$.
\end{proof}

\section{Proofs from section~\ref{sec:na}}\label{apn:NA_regret_general}
Before moving to the formal proof of Lemma~\ref{thm:MB_simple}, we recall the
following lemma about the ellipsoid algorithm, bounding the number of steps it
takes to find a feasible solution.

\begin{lemma}[Lemma 3.1.36 from \cite{GrotSchr1988}]\label{lem:ellipsoid}
	Given a full dimensional polytope $P = \{ x : Cx \leq d\} $, for $x\in
		\mathbb{R}^n$, and let $\langle C,d \rangle$ be the encoding length of
		$C$ and $d$.  If the initial ellipsoid is $E_0 = E(R^2
		I,0)$\footnote{$E(R,Z)$ indicates a ball of radius $R$ and center $Z$.}
		where $R =\sqrt{n}2^{\langle C,d\rangle-n^2}$ the ellipsoid algorithm
		finds a feasible solution after $O(n^2 \langle C, d \rangle)$ steps.
\end{lemma}

\noindent
Using the lemma above, we can now prove Lemma~\ref{thm:MB_simple}, which we also restate below.

\mblemma*
\begin{proof}
	Our analysis follows similarly to Theorem 3.2
	of~\cite{AlabKalaLigeMuscTzamVite2019}.  Initially observe that the only
	time we make a mistake is in the case with probability $(1-p_t)$ if the LP
	solution is not feasible.  Denote by $\mathcal{C}^*$ the set of the
		constraints of $\mathcal{LP}$ as defined in Algorithm~\ref{algo:base_NA}, and by  $\mathcal{C}_1 \subseteq
	\mathcal{C}_2\subseteq \ldots  \subseteq \mathcal{C}_t$ the constraint set
	for every round of the algorithm, for all $t\in [T]$. We also denote by
	$N_T(c)$ the number of times a constraint $c$ was not in $\mathcal{C}_t$
		for some time $t$ but was part of $\mathcal{LP}$.  Formally $N_T(c) =|\{c\in
	\mathcal{C}^*, c\not\in \mathcal{C}_t, \}|$ for a constraint $c\in
	\mathcal{C}^*$ and any $t\in [T]$. We can bound the mistake bound of
	Algorithm~\ref{algo:base_NA} as follows
	\begin{align*}
		\mathcal{M}(\mathcal{A}, T) \leq \sum_{c\in \mathcal{C}^*} \E{N_T(c)}
	.\end{align*}
	Let $t_c\in [T]$ be the round that constraint $c$ is added to the
	algorithm's constraint set for the first time, and let $S_c$ be the set of
	$\ell$ rounds in which we made a mistake because of this constraint.
	Observe that $\{S_1, S_2, \ldots S_\ell\} = S_c \subseteq \{\mathcal{C}_1,
	\mathcal{C}_2, \ldots \mathcal{C}_{t_c} \}$. We calculate the probability
	that $N_T(c)$ is incremented on round $k$ of $S$
		\[ {\small
		\Pr{}{N_T(c) \text{ incremented on round }k} = \prod_{i=1}^k (1-p_i) \leq (1-p)^k},
	\]
	since in order to make a mistake, we ended up on line~\ref{algo:step_else} of the algorithm. Therefore
	\[
			\E{N_T(c)} \leq \sum_{i=1}^T (1-p)^i  = \frac{(1-p)(1-(1-p))^T}{p}
	.\] 
	However in our case, every time a constraint is added to $\mathcal{C}_t$,
		one step of the ellipsoid algorithm is run, for the $\mathcal{LP}$. Using
		Lemma~\ref{lem:ellipsoid} and observing that in our case $\langle C, d
		\rangle = O(1)$ the total times this step can happen is at most
		$O(n^2)$, giving us the result of the lemma by setting
		$p=1/\sqrt{T}$.	
\end{proof}

\regretNA*
\begin{proof}[Proof of Theorem~\ref{thm:regret_NA}]
		Denote by $M$ be the mistake bound term, bounded in
		Lemma~\ref{thm:MB_simple}. Calculating the total average regret we get
	\begin{align*}
			\E{\text{Regret}_{NA}(\mathcal{A}_\mathcal{F}, T)} &+ \opt 
				= \frac{1}{T}\lp( M + \sum_{t=1}^{T} \E{S_t} \rp) & \text{Definition}\\
			& = \frac{1}{T} \lp( M + \sum_{t=1}^T p_t( n + \min_{i \in \mathcal{F}} c_i) 
								+ (1-p_t)\E{S_t}\rp) & \text{Algorithm}~\ref{algo:base_NA} \\ 
			& \leq \frac{1}{T} \lp( M + \sum_{t=1}^T p_tn + p_t \min_{i \in \mathcal{F}} c_i
						+ 2 \beta \opt_t \rp) &  \mathcal{A}_\mathcal{F},\text{ and ellipsoid's loss} \\
			& \leq  (2\beta + 1) \opt + \frac{1}{T}\lp(M +  n\sum_{t=1}^T p_t\rp) &  \min_{i\in \mathcal{F}} c_i \leq \opt_t \leq \opt\\
			& \leq  (2\beta +1) \opt + \frac{n}{\sqrt{T}} & 
		 	\sum_{t=1}^T p_t \leq \sqrt{T} \text{ from~\cite{AlabKalaLigeMuscTzamVite2019}}\\
			& \leq  (2\beta +1) \opt + O\lp(\frac{n^2}{\sqrt{T}}\rp) & \text{From Lemma~\ref{thm:MB_simple}} 
	.\end{align*}
		Therefore, subtracting $\opt$ from both sides we get the theorem.
\end{proof}


	\section{Rounding against Partially-adaptive} In this section we show how using the rounding
algorithms presented in \cite{ChawGergTengTzamZhan2020}, we obtain
$\alpha$-approximate rounding for our convex relaxations. We emphasize the fact
that these algorithms convert a fractional solution of the relaxed problem, to
an integer solution of comparable cost, without needing to know the scenario.
Formally we show the following guarantees.

\begin{corollary}[Rounding against PA]	\label{cor:a_rounding_pa} Given a
fractional solution $\ol{x}$ of cost $\ol{f}(\ol{x})$ there
exists an $\alpha$-approximate rounding where \begin{itemize} \item Selecting
						$1$ box: $\alpha=9.22$ (Using Lemma~\ref{lem:pa_1})
\item Selecting $k$ boxes: $\alpha=O(1)$ (Using Lemma~\ref{lem:pa_k}) 	\item
Selecting a matroid basis: $\alpha=O(\log k)$ (Using
Lemma~\ref{lem:pa_matroid}) 	\end{itemize} \end{corollary}

	In order to obtain the results of this corollary, we combine the ski-rental
		Lemma~\ref{thm:ski-rental} with Lemmas~\ref{lem:pa_1}, \ref{lem:pa_k}
		and \ref{lem:pa_matroid}, for each different case of constraints.
		Observe that, as we show in the following sections, the part where the
		fractional solution given is rounded to an integer permutation
		\textbf{does not depend on the scenario realized}. Summarizing the
		rounding framework of the algorithms is as follows.  \begin{enumerate}
				\item Receive fractional solution $\ol{x}, \ol{z}$.
				\item Use rounding~\ref{algo:pa_1}, \ref{algo:pa_k},
						\ref{algo:pa_matroid}, depending on the constraints, to
						obtain an (integer) permutation $\pi$ of the boxes.
		\item Start opening boxes in the order of $\pi$.  \item Use ski-rental
		to decide the stopping time.  \end{enumerate}

		\subsection{Rounding against Partially-adaptive for Choosing $1$
		Box}\label{apn:pa_1} The convex relaxation for this case (\ref{lp-spa})
		is also given in section~\ref{subsec:relaxation}, but we repeat it here for convenience.
	\begin{align*}
		\text{min}_{z \ge 0}  \quad & \sum_{i\in\boxes, 
		 t\in \timeset} (t +c_{i}^s) z_{it}^s  &  
		  \tag{Relaxation-SPA} \label{lp-spa}\\
		  \text{s.t.}\quad & \sum_{t\in\timeset, i\in \boxes}z_{it}^s = 1, & \\
		&  \hspace{0.7cm}  z_{it}^s \leq x_{it}, & i\in \boxes,t\in\timeset.
	\end{align*}

	Our main lemma in this case, shows how to obtain a constant competitive
		partially adaptive strategy, when given a scenario-aware solution.
\begin{lemma}\label{lem:pa_1} Given a scenario-aware fractional solution
$\ol{x}$ of cost $\ol{f}(\ol{x})$ there exists an efficient
partially-adaptive strategy $x$ with cost at most
$9.22\ol{f}(\ol{x})$.  \end{lemma}

\begin{proof} We explicitly describe the rounding procedure, in order to
		highlight its independence of the scenario realized. For the rest of
		the proof we fix an (unknown) realized scenario $s$.
		Starting from our (fractional) solution $\ol{\bm{x}}$ of cost $\ol{f}^s =
		\ol{f}_{o}^s + \ol{f}_{c}^s$, where $\ol{f}_{o}^s$ and
		$\ol{f}_{c}^s$ are the opening and values\footnote{Cost incurred by the
		value found inside the box.} cost respectively, we use the reduction in
		Theorem~5.2 in~\cite{ChawGergTengTzamZhan2020} to obtain a transformed
		fractional solution $\ol{\bm{x}}'$ of cost $\ol{f'}^s =
		\ol{f'}^s_o +
		\ol{f'}^s_c$. For this transformed solution, \cite{ChawGergTengTzamZhan2020} in Lemma~5.1 showed that
		\begin{equation}\label{eq:inspection_cost}
				\ol{f'}^s_o \leq \lp( \frac{\alpha}{\alpha-1}\rp)^2 \ol{f}_o^s
		\end{equation}
		for the opening cost and 
		\begin{equation}\label{eq:values_cost}
				\ol{f'}^s_c \leq \alpha \ol{f}^s_c
		\end{equation}
		for the cost incured by the value inside the box chosen. To achieve this, the initial variables $\ol{x}_{it}$ are scaled by a factor depending on $\alpha$ to obtain $\ol{x}_{it}'$.
For the remainder of the proof, we assume this scaling happened at the beginning, and abusing notation we denote by $\ol{\bm{x}}$ the scaled variables. This is without loss of
		generality since, at the end of the proof, we are taking into account the loss in cost incurred by the scaling (Inequalities~\ref{eq:inspection_cost} and 
		\ref{eq:values_cost}).
		The rounding process is
		shown in Algorithm~\ref{algo:pa_1}. 		\\

	\begin{algorithm}[H] \caption{ Scenario aware, $\alpha$-approximate
			rounding for $1$ box}\label{algo:pa_1} \KwData{Fractional solution
			$\bm{x}$ with cost $\ol{f} $, 
			set $\alpha=3 + 2\sqrt{2}$}
		\tcc{Part 1: Scenario-independent rounding} $\sigma :=$ for every $t
			=1,\ldots, n$, repeat \textbf{twice}: open each  box $i$
			w.p. $q_{it} =  \frac{\sum_{t'\le t} \ol{x}_{it'} }{t}$.
			\\ \mbox{}\\
		\tcc{Part 2: Scenario-dependent stopping time} 
			Given scenario $s$, calculate $\bm{z}^s$ and $\ol{f}_c^s$\\
			$\tau_s:=$ If box $i$ is opened and has value $c_i^s \leq \alpha
			\ol{f}_c^s$ then select it.
	\end{algorithm}
		\mbox{}\\

	The ratio of the opening cost of the integer to the fractional solution is bounded by
	\begin{align*}
			\frac{f_o^s}{\ol{f'}^s_o} & \leq \frac{2\sum_{t = 1}^\infty \prod_{k=1}^{t-1} 
				\lp( 1- \frac{\sum_{i\in A,t'\leq k}z_{it'}^s }{k}  \rp)^2}{\sum_i t\cdot z_{it}^s} & \text{Since } z_{it}^s\leq x_{it}\\
			& \leq \frac{2\sum_{t = 1}^\infty \exp\lp( -2\sum_{k=1}^{t-1} 
				\frac{\sum_{t'\leq k}z_{it'}^s }{k}  \rp) }{\sum_i t\cdot z_{it}^s} & \text{Using that }1+x\leq e^x\\
	\end{align*}

Observe that $h(\bm{z}) = \log \frac{f^s_o}{\ol{f'}^s_o} = \log f_o^s -
\log \ol{f}^s_o$ is a convex function since the first part is LogSumExp, and
$\log \ol{f}^s_o$ is the negation of a concave function. That means
$h(\bm{z})$ obtains the maximum value in the boundary of the domain, therefore
at the integer points where $z_{it}^s = 1$ iff $t=\ell$ for some $\ell\in [n]$,
otherwise $z_{it}^s = 0$. Using this fact we obtain

	\begin{align*}
	\frac{f_o^s}{\ol{f'}_o^s}& \leq \frac{ 2\ell + 2\sum_{t = \ell+1}^\infty \exp\lp( -2\sum_{k=\ell}^{t-1} \frac{1}{k}  \rp) }{\ell}
				& \text{Using that }z_{it}^s = 1 \text{ iff } t=\ell \\
			& =  \frac{2\ell + 2 \sum_{t=\ell+1}^\infty \exp\lp(H_{t-1} - H_{\ell-1} \rp) }{\ell}
					& H_{t} \text{ is the $t$'th harmonic number}\\
			& \leq \frac{ 2\ell + 2\sum_{t=\ell+1}^\infty \lp(\frac{\ell}{t} \rp)^2}{\ell} & 
					\text{Since }H_{t-1} - H_{\ell-1} \geq \int_\ell^t \frac{1}{x} dx = \log t - \log \ell\\
			& \leq \frac{ 2\ell + 2\ell^2 \int_\ell^\infty \frac{1}{t^2} dt }{\ell}
					& \text{Since }t^{-2}\leq x^{-2} \text{ for }x\in[t-1,t]\\
					& = 4.&
	\end{align*}
	Combining with equation~\ref{eq:inspection_cost}, we get that $f^s_o\leq
	4\lp( \frac{\alpha}{\alpha-1}\rp)^2 \ol{f}^s_o$. Recall that for the
	values cost, inequality~\eqref{eq:values_cost} holds, therefore requiring that
	$4 \lp(\frac{\alpha}{\alpha - 1}\rp)^2 = \alpha$, we have the lemma for
	$\alpha = 3 + 2\sqrt{2}$.
\end{proof}

\begin{corollary}
	For the case of MSSC, when the costs inside the boxes are either $0$ or
		$\infty$, the rounding of Lemma~\ref{lem:pa_1} obtains a $4$-approximation,
		improving the $11.473$ of~\cite{FotaLianPiliSkou2020}.
\end{corollary}
\subsection{Rounding against Partially-adaptive for Choosing $k$ Boxes}\label{apn:PA_k}
In this case we are required to pick  $k$ boxes instead of one. Similarly to
the one box case, we relax the domain $\ol{X} = \{x \in [0,1]^{n\times
n}: \sum_{i} x_{it} = 1 \text{ and } \sum_t x_{it} = 1\}$, to be the set of
doubly stochastic matrices. We define the relaxation $\ol{g}^s(\bm{x})$ as 
\begin{alignat}{3}
		\text{min}_{y\geq 0, z\geq 0}  & \quad\quad\quad \sum_{t\in \step}( 1-y_{t}^s)&
	+ &\quad  \sum_{i\in \boxes ,t\in \step} c_i^s z_{it}^s 
  	&\tag{Relaxation-SPA-k}\label{lp-pa-k} \\
	\text{subject to} &
		\hspace{2.5cm} z_{it}^s \quad  & \leq &\quad x_{it}, & \forall i\in \boxes,t\in\step \notag \\
   		& \hspace{1.2cm} \sum_{t'\leq t, i\not\in A}z_{it'}^s \quad & \geq &\quad (k-|A|)y_{t}^s,
	   		& \quad \quad \forall  A \subseteq \boxes, t\in \step \label{eq:pavspaKLP_cover}
\end{alignat}

Our main lemma in this case, shows how to obtain a constant competitive partially adaptive
strategy, when given a scenario-aware solution, in the case we are required to select $k$ items.
\begin{lemma}\label{lem:pa_k}
		Given a scenario-aware fractional solution $\ol{z}^s,\ol{x}$ of cost
		$\ol{g}^s(\ol{x})$ there exists an efficient partially-adaptive
	strategy $x$ with cost at most $O(1)\ol{g}^s(\ol{x})$.
\end{lemma}

\begin{proof}
	We follow exactly the steps of the proof of Theorem~6.2 from
	\cite{ChawGergTengTzamZhan2020}, but here we highlight two important
	properties of it.
	\begin{itemize}
			\item The rounding part that decides the permutation
					(Algorithm~\ref{algo:pa_k}) does \textbf{not} depend on the
					scenario realised, despite the algorithm being
					scenario-aware. 
			\item 	The proof does not use the fact that the initial solution
					is an \emph{optimal} LP solution. Therefore, the guarantee
					is given against \textbf{any} feasible fractional solution.
	\end{itemize}
	The rounding process is shown in Algorithm~\ref{algo:pa_k}.\\

	\begin{algorithm}[H]
			\caption{ Scenario aware, $\alpha$-approximate rounding for $k$-coverage from~\cite{ChawGergTengTzamZhan2020}}\label{algo:pa_k}
			\KwData{Solution $\bm{x}$ to \ref{lp-pa-k}. Set $\alpha=8$}
			\tcc{Part 1: Scenario-\textbf{independent} rounding}
		$\sigma:= $  For each phase $\ell=1,2,\ldots$,  
		open each box $i$ independently with probability $q_{i\ell} = \min \lp(\alpha \sum_{t\le 2^\ell} x_{it},1\rp)$. \\
		\mbox{}\\
		\tcc{Part 2: Scenario-dependent stopping time}
                $\tau_s:=$\\
		\quad Given scenario $s$, calculate  $\bm{y}^s, \bm{z}^s$\\
		\quad Define $t_s^*= \max\{t: y_{t}^s\leq 1/2\}$.\\
		\quad \If{$2^\ell\geq t^*_s$}{
			For each opened box $i$, select it with probability 
				$\min\left(\frac{\alpha\sum_{t\le 2^\ell} z_{it}^s}{q_{i\ell}},1\right)$.\\
			Stop when we have selected $k$ boxes in total.
		 \quad }
	\end{algorithm}
		\mbox{}\\
	Assume that we are given a fractional solution
	$(\ol{\bm{x}}, \ol{\bm{y}}^s, \ol{\bm{z}}^s)$, where $\ol{x}_{it}$ is the
	fraction that box $i$ is opened at time $t$, $\ol{z}_{it}^s$ is the fraction
	box $i$ is chosen for scenario $s$ at time $t$, and $\ol{y}_t^s$ is the
	fraction scenario $s$ is "covered`` at time $t$, where covered means that
	there are $k$ boxes selected for this scenario\footnote{We use the variables $y_t^s$ for
			convenience, they are not required since $y_{t}^s = \sum_{t'<t, i\in
		\boxes} z_{it}^s$.}. Denote by $\ol{f}_o^s$ ($f_o^s$) and $\ol{f}_c^s$  ($f_c^s$) the
	fractional (rounded) costs for scenario $s$ due to opening and selecting boxes
	respectively. 
	Denote also by $t^*_s$
	the last time step that $y_{t}^s \leq 1/2$ and observe that
\begin{equation}\label{eq:opt_t*}
		\ol{f}_{o}^s\geq \frac{t^*_s}{2}.
\end{equation}

Fix a realized scenario $s$ and denote by $\ell_0=\lceil\log t^*_s\rceil$.
		Using that for each box $i$ the probability that it is
        selected in phase $\ell\ge\ell_0$ is
        $\min(1,8\sum_{t'\leq 2^{\ell}}z_{it'}^s)$, we use the following lemma 
        from~\cite{BansGuptRavi2010} that still holds in our case ; the proof of the lemma
		only uses constraint~\ref{eq:pavspaKLP_cover} and a Chernoff bound.

		\begin{lemma}[Lemma 5.1 in \cite{BansGuptRavi2010}]\label{lem:successprobk}
	If each box $i$ is selected w.p. at least $\min(1,8\sum_{t'\leq t}z_{it'}^s)$
	for $t\geq t_s^*$, then with probability at least $1-e^{-9/8}$, at least $k$ different boxes are
	selected.	
	\end{lemma}

	Similarly to \cite{ChawGergTengTzamZhan2020}, let $\gamma=e^{-9/8}$ and
	$\boxes_j$ be the set of boxes selected at phase $j$. Since the number of
	boxes opened in a phase is independent of the event that the algorithm
	reaches that phase prior to covering scenario $s$ the expected inspection cost is
	\begin{align*}
			\E{f_{o}^s \text{ after phase } \ell_0} 	& = \sum_{\ell=\ell_0}^\infty
		\E{f_{o}^s\text{ in phase } \ell} \cdot  \Pr{}{\text{reach phase } \ell}  & \\
	   	& \leq  \sum_{\ell = \ell_0}^\infty \sum_{i\in\boxes} \alpha \sum_{t'\leq 2^{\ell}}x_{it'} \cdot
			\prod_{j=\ell_0}^{\ell-1} \Pr{}{|\boxes_j| \leq k} \\
		& \leq \sum_{\ell = \ell_0}^\infty  2^{\ell}\alpha\cdot \gamma^{\ell-\ell_0} 
		& \text{Lemma \ref{lem:successprobk}, }x_{it} \text{ doubly stochastic}\\
   	& = \frac{2^{\ell_0}\alpha}{1-2\gamma}<\frac{2t_s^*\alpha}{1-2\gamma}\leq
		\frac{4\alpha\ol{f}^s_o}{1-2\gamma}.& \ell_0=\lceil\log
		t_s^*\rceil \text{ and ineq. \eqref{eq:opt_t*}}
	\end{align*}

	Observe that the expected opening cost at each phase $\ell$ is at most
	$\alpha2^{\ell}$, therefore the expected opening cost before phase $\ell_0$
	is at most
	$\sum_{\ell<\ell_0}\alpha2^\ell<2^{\ell_0}\alpha<2t_s^*\alpha\leq
	4\alpha\ol{f}_{o}^s$. Putting it all together, the total expected opening
	cost of the algorithm for scenario $s$ is
	\[
			f_{o}^s\leq 4\alpha\ol{f}_{o}^s+\frac{4\alpha\ol{f}_{o}^s}{1-2\gamma}<123.25\ol{f}_{o}^s.
	\]

	To bound the cost of our algorithm, we find the expected total value of any
	phase $\ell$, conditioned on selecting at least $k$ distinct boxes in this
	phase.
	\begin{align*}
			\mathbb{E} [ \text{cost in phase }\ell &|  \text{at least $k$ boxes are selected in phase }\ell]  \\
			   & \leq \frac{\E{\text{cost in phase }\ell}}{\Pr{}{\text{at least $k$ boxes are selected in phase }\ell}}\\
			& \leq \frac{1}{1-\gamma}\E{\text{cost in phase }\ell}\\
			& \leq \frac{1}{1-\gamma}\sum_{i\in\boxes}\alpha\sum_{t\leq 2^{\ell}}z_{it}^sc_{i}^s
			=\frac{1}{1-\gamma} \alpha\ol{f}_{c}^s < 11.85 \ol{f}_{c}^s.
	\end{align*}

	The third line follows by Lemma \ref{lem:successprobk} and the
	last line by the definition of $\ol{f}_{c}^s$. Notice that the
        upper bound does not depend on the phase $\ell$, so the same
        upper bound holds for $f_{c}^s$. Thus the total cost
        contributed from scenario $s$ in our algorithm is
		\[f^s = f_{o}^s + f_{c}^s<123.25 \ol{f}_{o}^s+11.85\ol{f}_{c}^s\leq 123.25 \ol{f}^{s},\]
		which gives us the lemma.

\end{proof}

\subsection{Rounding against Partially-adaptive for Choosing a Matroid Basis}\label{apn:PA_matroid}
Similarly to the $k$ boxes case, we relax the domain $\ol{X} = \{x \in
[0,1]^{n\times n}: \sum_{i} x_{it} = 1 \text{ and } \sum_t x_{it} = 1\}$, to be
the set of doubly stochastic matrices. Let $r(A)$ for any set $A\subseteq
\boxes$ denote the rank of this set.  We define $\ol{g}^s(\bm{x})$ as 
	\begin{alignat}{3}
		\text{min}_{y\geq 0, z\geq 0}  &
		\sum_{t\in \step}(1 - y_{t}^s) &\quad  + \quad &\sum_{i\in \boxes, t\in \step}c_i^s z_{it}^s  \tag{Relaxation-SPA-matroid}\label{lp-pa-matroid} \\
		\text{subject to} \quad  &  \hspace{0.7cm} \sum_{t\in \step, i\in
                             A}z_{it}^s & \quad \leq	\quad &  r(A),
                           & \forall  A\subseteq
                           \boxes \label{eq:LP_mat_less_than_rank}\\ 
			   & \hspace{2cm} z_{it}^s 	  & \quad \leq \quad & x_{it}, & \forall i\in \boxes ,t\in \step \notag\\ 
			   &\hspace{0.65cm} \sum_{i\not\in A}\sum_{t'\leq t}z_{it'}^s &	 \quad \geq \quad & (r([n])-r(A))y_{t}^s,\quad 
				   & \forall A\subseteq \boxes, 
                                   t\in
                                   \step  \label{eq:LP_mat_subsets_rank} 
	\end{alignat}

	Our main lemma in this case, shows how to obtain a constant $\log k$-competitive partially adaptive
strategy, when given a scenario-aware solution, in the case we are required to select a matroid basis.
\begin{lemma}\label{lem:pa_matroid}
		Given a scenario-aware fractional solution $\ol{z}^s,\ol{x}$ of cost
	$\ol{g}^s(\ol{x})$ there exists an efficient partially-adaptive
	strategy $z_s$ with cost at most $O(\log k)\ol{g}^s(\ol{x})$.
\end{lemma}

\begin{proof}[Proof of Lemma~\ref{lem:pa_matroid}]
	We follow exactly the steps of the proof of Lemma~6.4
	from~\cite{ChawGergTengTzamZhan2020}, but similarly to the $k$ items case
	we highlight the same important properties; (1) the rounding that decides
		the permutation \textbf{does not depend on the scenario} (2) the proof does not use
	the fact that the initial solution given is optimal in any way.	The
	rounding process is shown in Algorithm~\ref{algo:pa_matroid}.\\

	\begin{algorithm}[H]
		\caption{Scenario aware, $O(\log n)$-approximate rounding, matroids from \cite{ChawGergTengTzamZhan2020}}\label{algo:pa_matroid}
		\KwData{Fractional solution $\bm{x}, \bm{y}, \bm{z}$ for scenario $s$, $\alpha = 64$.}
		\tcc{Part 1: Scenario-independent rounding}
		$\sigma :=$ for every $t =1,\ldots, n$, open each box $i$ independently
			with probability $q_{it} = \min \lp\{\alpha \ln k\frac{\sum_{t'\le t}
			x_{it'} }{t},1\rp\}$. \\
			\mbox{}\\
		\tcc{Part 2: Scenario-dependent stopping time.}
                $\tau_s:=$\\
				\quad Given scenario $s$, calculate $\bm{y}, \bm{z}$\\
				\quad Let $t_s^*= \min\{t: y_{t}^s\leq 1/2\}$.\\
		\quad \If{$t>t^*_s$}{
			For each opened box $i$, select it with probability $\min \lp\{\frac{\alpha\ln k\sum_{t'\le t} z_{it'}^s}{t q_{it}} ,\ 1\rp\}$.\\
			Stop when we find a base of the matroid.
		}
	\end{algorithm}
		\mbox{}\\
	Denote by $(\ol{\bm{x}}, \ol{\bm{y}}^s, \ol{\bm{z}}^s)$ the fractional
	scenario-aware solution given to us, where $\ol{x}_{it}$ is the fraction
	that box $i$ is opened at time $t$, $\ol{z}_{it}^s$ is the fraction box $i$
	is chosen for scenario $s$ at time $t$, and $\ol{y}_t^s$ is the fraction
	scenario $s$ is "covered`` at time $t$, where covered means that there is a matroid basis selected for this scenario.
	Denote also by $\ol{f}_o^s\ (f_o^s), \ol{f}_c^s\ (f_c^s)$ and
	the fractional costs for scenario $s$ due to opening
	and selecting boxes for the fractional (integral) solution respectively.
	
	In scenario $s$, let phase $\ell$ be when $t\in (2^{\ell-1}t_s^*, 2^\ell t_s^*]$.
	We divide the time after $t^*_s$ into exponentially increasing
	phases, while in each phase we prove that our success probability is a constant. 
	The following lemma gives an upper bound for the opening cost needed 
	in each phase to get a full rank base of the matroid, and still holds in
	our case, since only uses the constraints of a feasible solution.

	\begin{lemma}[Lemma 6.6 from~\cite{ChawGergTengTzamZhan2020}]\label{lem:rank}
		In phase $\ell$, the expected number of steps needed to select a set of full rank is at most
		$(4+2^{\ell+2} /\alpha)t^*_s$.
	\end{lemma}

	Define $\mathcal{X}$ to be the random variable indicating number of steps needed 
	to build a full rank subset. The probability that we build a full rank basis 
	within some phase $\ell\geq 6$ is
	\begin{align}\label{eq:prob_cont}
			\Pr{}{\mathcal{X} \leq 2^{\ell-1}t_s^*} 
				 \ge  1 - \frac{\E{\mathcal{X}}}{2^{\ell-1}t_s^*} 
				 \ge  1 - \frac{1}{2^{\ell-1}t_s^*}(4+2^{\ell+2} /\alpha)t^*_s 
				=1-2^{3-\ell}-\frac{8}{\alpha}\geq \frac{3}{4},
	\end{align}
	where we used Markov's inequality for the first inequality and
	Lemma~\ref{lem:rank} for the second inequality.  To calculate the total
	inspection cost, we sum up the contribution of all phases. 
	\begin{align*}
		\E{f_{o}^s \text{ after phase 6}} &= \sum_{\ell=6}^\infty \E{f_{o}^s\text{ at phase } \ell} 
				\cdot  \Pr{}{\alg\text{ reaches phase } \ell} \\
			& \leq \sum_{\ell = 6}^\infty  \sum_{t=2^{\ell-1}t_s^*+1}^{2^{\ell}t_s^*}\sum_{i\in\boxes}\alpha\ln k\cdot\frac{\sum_{t'\leq t}x_{it'}}{t}\lp(\frac{1}{4}\rp)^{\ell-6} 
										  	& \text{Algorithm~\ref{algo:pa_matroid}} \\
			& \leq \sum_{\ell = 6}^\infty  2^{\ell-1}t_s^*\alpha\ln k\cdot \lp(\frac{1}{4}\rp)^{\ell-6}  
						& x_{it} \text{ doubly stochastic}\\
			& = \frac{128\alpha\ln kt_s^*}{3}\leq\frac{256c\ln k\ol{f}_o^s}{3}. & \text{Since } t_s^*\leq 2\ol{f}_o^s
	\end{align*}

	Since the expected opening cost at
		each step is $\alpha\ln k$ and there are $2^5t_s^*\leq 64\ol{f}^s_o$
        steps before phase $6$, we have
	\[
			f_o^s\leq \alpha\ln k\cdot 64\ol{f}_o^s+\frac{256\alpha\ln k\ol{f}_o^s}{3} 
			= O(\log k)\ol{f}_o^s.
	\]

	Similarly to the $k$-coverage case, to bound the cost of our algorithm,
        we find the expected total cost of any phase $\ell\geq 6$,
        conditioned on boxes forming a full rank base are selected in
        this phase.
	\begin{align*}
			\textbf{E} [ f_c^s\text{ in phase }\ell| & \text{full rank base selected in phase }\ell]\\
		& \leq \frac{\E{f_c^s \text{ in phase }\ell}}{\Pr{}{\text{full rank base selected in phase }\ell}}\\
		&\leq \frac{1}{3/4}\E{f_c^s\text{ in phase }\ell}\\
		&\leq \frac{1}{3/4}\sum_{i\in\boxes}\sum_{t=2^{\ell-1}t_s^*+1}^{2^{\ell}t_s^*}\alpha\ln k
				\frac{\sum_{t'\leq t}z_{it'}^s c_i^s}{t}\\
		&\leq \frac{1}{3/4}\sum_{t=2^{\ell-1}t_s^*+1}^{2^{\ell}t_s^*}\alpha\ln k\sum_{i\in\boxes}
				\frac{\sum_{t'\in\step}z_{it'}^s c_i^s}{2^{\ell-1}t_s^*}\\
		&=\frac{1}{3/4}\alpha \ln k\ol{f}_c^s
		=O(\log k)\ol{f}_c^s.
	\end{align*}
	Such upper bound of conditional expectation does not depend on $\ell$, thus
	also gives the same upper bound for $f_c^s$. Therefore
	$f^s=f_o^s + f_c^s \leq O(\log k)(\ol{f}_o^s+\ol{f}_c^s)=O(\log
	k)\ol{f}^s$. 

\end{proof}

	\section{Linear Programs \& Roundings against NA}\label{apn:NA}
\subsection{Competing with the non-adaptive for choosing $1$ box} 
The linear program for this case (\ref{lp-na}) is already given in the
preliminaries section. The result in this case is a $e/(e-1)$-approximate
partially adaptive strategy, given in \cite{ChawGergTengTzamZhan2020} is
formally restated below, and the rounding algorithm is presented in Algorithm~\ref{algo:na-1}.

\begin{theorem}[Theorem 4.2 from~\cite{ChawGergTengTzamZhan2020}]\label{thm:na_1}
	There exists an efficient partially adaptive algorithm with cost at most
		$e/(e-1)$ times the total cost of the optimal non-adaptive strategy.
\end{theorem}

	\begin{algorithm}[H]
			\caption{SPA vs NA from\cite{ChawGergTengTzamZhan2020}}\label{algo:na-1}
		 \KwIn{ Solution $\bm{x}, \bm{z}$ to program~\eqref{lp-na}; scenario $s$}
		$\sigma:=$ For $t\geq 1$, select and open box $i$
                with probability $\frac{x_i}{\sum_{i\in \boxes}x_i}$.\\
	$\tau_s:=$ If box $i$ is opened at step $t$, select the box and stop with probability $\frac{z_{i}^s}{x_i}$.
	\end{algorithm}

\subsection{Competing with the non-adaptive benchmark for choosing $k$ boxes}\label{apn:NA_k}

We move on to consider the case where we are required to pick $k$ distinct
boxes at every round. Similarly to the one box case, we define the optimal
non-adaptive strategy that can be expressed by a linear program. We start by
showing how to perform the rounding step of line~\ref{algo:na_round} of
Algorithm~\ref{algo:base_NA} in the case we have to select $k$ boxes. The guarantees are given in
Theorem~\ref{thm:PA_vs_NA_k} and the rounding is presented in
Algorithm~\ref{algo:na-k}. This extends the results
of~\cite{ChawGergTengTzamZhan2020} for the case of selecting $k$ items against
the non-adaptive.

\begin{lemma}\label{lem:SPA_vs_NA_k}
		There exists a scenario-aware partially adaptive $4$-competitive
	algorithm to the optimal non-adaptive algorithm for picking $k$ boxes.
\end{lemma}

Combining this lemma with Theorem~3.4 from \cite{ChawGergTengTzamZhan2020} we
get Theorem~\ref{thm:PA_vs_NA_k}.

\begin{theorem}\label{thm:PA_vs_NA_k}
	We can efficiently find a partially-adaptive strategy for optimal search
		with $k$ options that is $4e/(e-1)$-competitive against the optimal
	non-adaptive strategy.
\end{theorem}

Before presenting the proof for Lemma~\ref{lem:SPA_vs_NA_k}, we formulate our
problem as a linear program as follows. The formulation is the same
as~\ref{lp-na}, we introduce constraints~\ref{eq:mLP_NA_select_k}, 
since we need to pick $k$ boxes instead
of $1$.
\begin{alignat}{3}
  \text{minimize}  \quad & \sum_{i\in \boxes} x_i   & \quad  + \quad
  & \frac{1}{|\scenario|}\sum_{i\in \boxes,s\in \scenario}c_{i}^s z_{i}^s
  & \tag{LP-NA-k} \label{lp-na-k}\\
  \text{subject to}\quad 
  	& \sum_{i\in \boxes}z_{i}^s & \quad  = \quad   & k, & \forall s\in \scenario \label{eq:mLP_NA_select_k}\\
		 & \hspace{0.7cm} z_{i}^s 			 & \quad \leq \quad & x_i , & \forall i\in \boxes, s\in \scenario \notag\\
		 &  \hspace{0.2cm} x_i, z_{i}^s 		 &\quad \in \quad & [0,1] & \forall i\in \boxes, s\in \scenario \notag
\end{alignat}

Denote by $\opt_p = \sum_{i\in \boxes} x_i$ and $\opt_c  = 1/|\scenarios|
\sum_{i\in \boxes, s\in \scenarios} c_{i}^s z_{i}^s$ to be the optimal opening
cost and selected boxes' costs, and respectively $\alg_p$ and $\alg_c$ the algorithm's costs.\\

\newcommand\xlow{\mathcal{X}_{\text{low}}}
\begin{algorithm}[H]
	\caption{ SPA vs NA, k-coverage}\label{algo:na-k}
		\KwIn{Solution $\bm{x}, \bm{z}$ to above \ref{lp-na-k}, scenario $s$. We set $\beta=1/100, \alpha=1/4950$}
		Denote by $\xlow = \{ i: x_i< 1/\beta\}$ and $X = \sum_{i\in \xlow} x_i$\\
		$\sigma:= $ open all boxes that $x_i \geq 1/\beta$, from $\xlow$ select each box $i$ 
			w.p. $\frac{x_i}{X}$\\
			\mbox{}\\
Denote by $k'$ and $\opt_c'$ the values of $\opt_c$ and $k$ restricted in the set $\xlow$\\
            $\tau_s:=$ select all boxes that $z_{i}^s\geq 1/\beta$ \\
		\quad \quad Discard all boxes $i$ that $c_i > \alpha \opt_c'/k'$\\
				\quad\quad From the rest select box $i$ with probability $\frac{x_i}{X}$\\
		\quad \quad Stop when we have selected $k$ boxes in total.
	 
\end{algorithm}

\begin{proof}[Proof of Lemma~\ref{lem:SPA_vs_NA_k}]
Let $(\bm{x}, \bm{z})$ be the solution to \eqref{lp-na-k}, for some scenario
$s\in \scenarios$. We round this solution through the following steps, bounding the extra
cost occurred at every step. Let $\beta>1$ be a constant to be set later.

	\begin{itemize}
		\item \textbf{Step 1}: open all boxes $i$ with $x_i\geq 1/\beta$, select
				all that $z_{i}^s \geq 1/\beta$. This step only incurs at most
				$\beta(\opt_p + \opt_c)$ cost. The algorithm's value cost is 
				$\alg_c = \sum_{i:z_{i}^s\geq 1/\beta} c_i$ while $\opt_c = \sum_{i}z_{i}^sc_i \geq
				\sum_{i: z_{i}^s\geq 1/\beta} c_i z_{i}^s \geq 1/\beta \sum_{i:z_{i}^s\geq
				1/\beta} c_i = 1/\beta \alg_c$. A similar argument holds for the opening cost.
		\item \textbf{Step 2}: let $\mathcal{X}_{\text{low}} = \{i: x_{i} <1/\beta\}$, 
				and denote by $\opt_c'$ and $k'$ the new
				values for $\opt_c$ and $k$ restricted on the set
				$\mathcal{X}_{\text{low}}$ and by $X = \sum_{i\in
				\mathcal{X}_{\text{low}}} x_i$.
				\begin{itemize}
					\item \textbf{Step 2a}: convert values to either $0$ or
							$\infty$ by setting $c_i = \infty$ for every box
							$i$ such that $c_i > \alpha \opt_c'/k'$ and denote by
							$\mathcal{L}_s =  \{i: c_i\leq \alpha \opt_c'/k' \}$. 
					\item \textbf{Step 2b}: select every box with probability $
							\frac{x_i}{ X }$, choose a box only if it is in $\mathcal{X}_{\text{low}}$. 
							Observe that the probability of choosing the $j$'th
							box from $L_s$ given that we already have chosen
							$j-1$ is %
					\begin{align*}
						\Pr{}{\text{choose }j\text{'th}|\text{have chosen }j-1} & 
						\geq \frac{\sum_{i\in L_s}x_i - j/\beta}{X} 
								& \text{Since }x_i\leq 1/\beta \text{ for all }x_i\in\mathcal{X_{\text{low}}}\\
						 &\geq \frac{\sum_{i\in L_s}z_{i}^s - j/\beta}{X} & \text{From LP constraint} \\
						 &\geq \frac{(1-1/\alpha)k' -j/\beta}{\opt_p'}& \text{From Markov's Inequality}\\
						 &\geq \frac{(1-1/\alpha)k' -k'/\beta}{\opt_p'}& \text{Since } j\leq k'\\
						 &\geq \frac{(\alpha\beta - \beta - \alpha)k'}{\alpha\beta \opt_p'}&
					\end{align*}
					Therefore the expected time until we choose $k'$ boxes is 
					\begin{align*}
						\E{\alg_p} &= \sum_{j=1}^{k'} \frac{1}{\Pr{}{\text{choose }j\text{'th}|\text{have chosen }j-1}}\\
								   &\leq \sum_{j=1}^{k'} \alpha \beta \frac{\opt_p'}{(\alpha\beta - \alpha - \beta)k'}&\\
									& = \alpha\beta \frac{\opt_p'}{\alpha\beta -\alpha - \beta}
					\end{align*}

				\end{itemize}
				Observe also that since all values selected are are $c_i \leq
				\alpha \opt_c'/k'$, we incur value cost $\alg_c\leq
				\alpha \opt_c'$.
		\end{itemize}
		Putting all the steps together, we get $\alg \leq \lp(\beta+\frac{\alpha\beta}{\alpha\beta - \alpha - \beta} \rp) 
		\opt_p + (\beta + \alpha) \opt_c \leq 4\opt $, when setting $a= 2\beta/(\beta-1)$ and $\beta= 1/100$
\end{proof}


\subsection{Competing with the non-adaptive benchmark for choosing a matroid basis}\label{apn:NA_matroid}

In this section $\mathcal{F}$ requires us to select a basis of a given matroid. More
specifically, assuming that boxes have an underlying matroid structure we seek
to find a basis of size $k$ with the minimum cost and the minimum query time.
Let $r(A)$ denote the rank of the set $A\subseteq \boxes$.  Using the linear
program of the $k$-items case, we replace the constraints to ensure that ensure
that we select at most $r(A)$ number of elements for every set and that
whatever set $A$ of boxes is already chosen, there still enough elements to
cover the rank constraint. The guarantees for this case are given in
Theorem~\ref{thm:PA_vs_NA_matroid} and the rounding presented in
Algorithm~\ref{algo:na-matroid}. This case also extends the results
of~\cite{ChawGergTengTzamZhan2020}.

\begin{lemma}\label{lem:PA_vs_NA_matroid}
	There exists a scenario-aware partially-adaptive $O(\log k)$-approximate
	algorithm to the optimal non-adaptive algorithm for picking a matroid
	basis of rank $k$.
\end{lemma}
Combining this lemma with Theorem~3.4 from \cite{ChawGergTengTzamZhan2020} we
get Theorem~\ref{thm:PA_vs_NA_matroid}.

\begin{theorem}\label{thm:PA_vs_NA_matroid}
	We can efficiently find a partially-adaptive strategy for optimal search
	over a matroid of rank $k$ that is $O(logk)$-competitive against the
	optimal non-adaptive strategy.
\end{theorem}

In order to present the proof for Lemma~\ref{lem:PA_vs_NA_matroid}, we are using
the LP formulation of the problem with a matroid constraint, as shown
below. Let $r(A)$ denote the rank of the set $A\subseteq \boxes$. 
The difference with \ref{lp-na-k} is that we replace constraint~\ref{eq:mLP_NA_select_k} with
constraint\ref{eq:mLP_NA_select_matroid} which ensures we select at most $r(A)$
number of elements for every set and constraint \eqref{lp-na-matr-subsets}
ensures that whatever set $A$ of boxes is already chosen, there still enough
elements to cover the rank constraint. 

\begin{alignat}{3}
  \text{minimize}  \quad & \sum_{i\in \boxes} x_i   & \quad  + \quad
  & \frac{1}{|\scenario|}\sum_{i\in \boxes,s\in \scenario}c_{i}^s z_{i}^s
  & \tag{LP-NA-matroid} \label{lp-na-matroid}\\
  \text{subject to}\quad & \sum_{i\in \boxes}z_{i}^s & \quad  \leq \quad   & r(A), 	& \forall s\in \scenario , A\subseteq \boxes \label{eq:mLP_NA_select_matroid}\\
						 & \sum_{i\in A} z_{i}^s & \quad \geq \quad & r([n]) - r(A) & \forall A\subseteq \boxes, \forall s\in \scenario \label{lp-na-matr-subsets}\\
		& \hspace{0.7cm} z_{i}^s 		 & \quad \leq \quad & x_i , & \forall i\in \boxes, s\in \scenario  \label{eq:mLP_NA_matr_zx}\\
	 &  \hspace{0.2cm} x_i, z_{i}^s 		 &\quad \in \quad & [0,1] & \forall i\in \boxes, s\in \scenario \notag
\end{alignat}

Similarly to the case for $k$ items, denote by $\opt_p = \sum_{i\in \boxes}
x_i$ and $\opt_c = 1/|\scenarios|\sum_{i\in \boxes, s\in \scenarios} c_{i}^s z_{i}^s$, and
$\alg_p, \alg_c$ the respective algorithm's costs.\\

\begin{algorithm}[H]
	\caption{ SPA vs NA, matroid}\label{algo:na-matroid}
		\KwIn{Solution $\bm{x}, \bm{z}$ to above \ref{lp-na-matroid}, scenario $s$. We set $\beta=1/100, \alpha=1/4950$}
		Denote by $\xlow = \{ i: x_i< 1/\beta\}$ and $X = \sum_{i\in \xlow} x_i$\\
		$\sigma:= $ open all boxes that $x_i \geq 1/\beta$, from $\xlow$ select each box $i$ 
			w.p. $\frac{x_i}{X}$\\
			\mbox{}\\
Denote by $k^j$ and $\opt_c^j$ the values of $\opt_c$ and $k$ restricted in the set $\xlow$ when $j$ boxes are selected.\\
            $\tau_s:=$ select all boxes that $z_{i}^s\geq 1/\beta$ \\
		\quad \quad Discard all boxes $i$ that $c_i > \alpha \opt_c^j/k^j$\\
				\quad\quad From the rest select box $i$ with probability $\frac{x_i}{X}$\\
		\quad \quad Stop when we have selected $k$ boxes in total.
	 
\end{algorithm}

\begin{proof}[Proof of Lemma~\ref{lem:PA_vs_NA_matroid}]
		Similarly to Lemma~\ref{lem:SPA_vs_NA_k}, let $(\bm{x}, \bm{z})$ be the
		solution to \ref{lp-na-matroid}, for some scenario $s\in \scenarios$.
		We round this solution through the following process. 
	Let $\beta>1$ be a constant to be set later.

	\begin{itemize}
		\item \textbf{Step 1}: open all boxes $i$ with $x_i\geq 1/\beta$, select
				all that $z_{i}^s \geq 1/\beta$. This step only incurs at most
					$\beta(\opt_p + \opt_c)$ cost.
		\item \textbf{Step 2}: let $\xlow = \{i: x_{i}
				<1/\beta\}$. Denote by $\opt_c'$ and $k'$ the new values of
					$\opt_c$ and $k$ restricted on $\xlow$.
					At every step, after having selected $j$ boxes, we restrict
					our search to the set of low cost boxes $\mathcal{L}_s^j =
					\{i: v_i\leq \alpha \opt_c^j/k^j \}$ where $\opt_c^j$ and
					$k^j$ are the new values for $\opt_c$ and $k$ after having
					selected $k^j = j$ boxes.

				\begin{itemize}
					\item \textbf{Step 2a}: Convert values to either $0$ or
							$\infty$ by setting $v_i = \infty$ for every box
							$i$ such that $v_i > \alpha \opt_c^j/k^j $.
					\item \textbf{Step 2b}: Select every box with probability $
							\frac{x_i}{ X }$, choose a box only if it is in $\mathcal{X}_{\text{low}}$. 
							Observe that the probability of choosing the $j$'th
							box from $L_s$ given that we already have chosen
							$j-1$ is %
					\begin{align*}
						\Pr{}{\text{choose }j\text{'th}|\text{have chosen }j-1} & 
							\geq \frac{\sum_{i\in L_s^{j-1}}x_i}{X} & \\
							&\geq \frac{\sum_{i\in L_s^{j-1}}z_{i}^s }{X} & \text{From LP constraint \eqref{eq:mLP_NA_matr_zx}} \\
							&\geq \frac{k - (k-j)}{X}& \text{From LP constraint \eqref{lp-na-matr-subsets}}\\
							&= \frac{j}{\opt_p'}& 
					\end{align*}
					Therefore the expected time until we choose $k'$ boxes is 
					\begin{align*}
						\E{\alg_c} &= \sum_{j=1}^{k'} \frac{1}{\Pr{}{\text{choose }j\text{'th}|\text{have chosen }j-1}}\\
								   & \leq \opt_p' \sum_{j=1}^{k'} \frac{1}{j}\\
								&\leq \log k\cdot  \opt_p
					\end{align*}

				\end{itemize}
					Observe also that every time we choose a value from the set
					$\mathcal{L}_s^j$, therefore the total cost incurred by the
					selected values is
					\[
						\alg_v \leq \sum_{i=1}^{k'} \alpha \frac{\opt_c^i}{k_i} 
							\leq \sum_{i=1}^{k'}  \frac{\opt_c}{i} 
							\leq  \alpha \log k \cdot \opt_c 
					\]
		\end{itemize}
		Putting all the steps together, we get $\alg \leq O(\log k) \opt $
\end{proof}


\end{document}